\documentclass{article}
\usepackage{ijcai20}

\usepackage{times}
\usepackage{soul}
\usepackage{url}
\usepackage[utf8]{inputenc}
\usepackage[small]{caption}
\usepackage{graphicx}
\usepackage{amsmath}
\usepackage{amsthm}
\usepackage{booktabs}
\usepackage{enumitem}

\usepackage{multirow}
\usepackage[ruled,vlined]{algorithm2e} % For algorithms

\newtheorem{theorem}{Theorem}

\title{Federated Extra-Trees with Privacy Preserving}

\author{
Yang Liu$^{1,2}$
\and
Mingxin Chen$^{3}$
\and
Wenxi Zhang$^{1,2}$\and
Junbo Zhang$^{1,2,3}$ \and 
Yu Zheng$^{1,2,3}$
\affiliations
$^1$JD Intelligent Cities Research, JD Digits, Beijing, China\\
$^2$JD Intelligent Cities Business Unit, JD Digits, Beijing, China\\
$^3$Institute of Artificial Intelligence, Southwest Jiaotong University, Chengdu, China
\emails
\{liuyang21cn, msjunbozhang, msyuzheng\}@outlook.com, \{mxchen1997, zhangwenxi7\}@gmail.com
}
\begin{document}

\maketitle

\begin{abstract}
It is commonly observed that the data are scattered everywhere and difficult to be centralized. The data privacy and security also become a sensitive topic. The laws and regulations such as the European Union's General Data Protection Regulation (GDPR) are designed to protect the public's data privacy. However, machine learning requires a large amount of data for better performance, and the current circumstances put deploying real-life AI applications in an extremely difficult situation. To tackle these challenges, in this paper we propose a novel privacy-preserving federated machine learning model, named \textit{Federated Extra-Trees}, which applies local differential privacy in the federated trees model. A secure multi-institutional machine learning system was developed to provide superior performance by processing the modeling jointly on different clients without exchanging any raw data. We have validated the accuracy of our work by conducting extensive experiments on public datasets and the efficiency and robustness were also verified by simulating the real-world scenarios. Overall, we presented an extensible, scalable and practical solution to handle the \textit{data island} problem.
\end{abstract}

\section{Introduction\label{sec:intro}}
Although we are living in the era of \textit{Big Data}, we often have to face the fact that there are not enough data for modeling.
Except for those data-rich companies, most organizations don't own enough data to serve their academic research or business projects, and the necessary data are scattered across different organizations not shared. 
Because of the serious \textit{data island} situations, secure multi-institutional collaborative modeling has many important potential applications, such as medical study, target marketing, risk management, etc. 
In the work of \cite{Sheller_2019}, the researchers built a semantic segmentation model on multimodal brain scans. 
The entire modeling was conducted on a multi-institutional collaboration and no raw patient data were shared.

However, it is still challenging to unite multiple institutions modeling together. 
One of the biggest concerns is data privacy and protection. 
Not long ago, the Federal Trade Commission (FTC) of the United States imposed a record-breaking \textdollar 5 billion penalty to Facebook, due to its violation of an FTC's 2012 order about user data privacy. 
% Not long ago, Facebook received a record-breaking \textdollar 5 billion penalty from the U.S. governemnt, due to its violation of an FTC's 2012 order about user data privacy. 
Companies in many other areas also face similar legal sanctions. The enactment of laws and regulations such as the European Union's General Data Protection Regulation (GDPR)
% and China's Cyber Security Law 
has made the cross-institutional data mining and modeling more difficult.

% This is not the first time that Facebook has violated the regulations. 
% In 2018, the Information Commissioners Office (ICO) of United Kingdom (UK) also fined this company \textsterling 500,000 over the scandal with the Cambridge Analytica, and this penalty is the maximum possible amount by UK's laws. 
% Many other companies also face similar legal sanctions. In 2018, due to separate data breaches, Marriott paid \textsterling 99,200,396 to the ICO and British Airways was fined \textsterling 183.39 million. 
% Recently, Google, Amazon and Apple were all reported about potential misuage and leakage of user's voice data collected from their voice assistant products. 
% Although they stated the data were not linked to the original users and were only used to improve their products, under the European Union's (EU) General Data Protection Regulation (GDPR), this kind of processing on users' data could have already violated the regulations.

To meet the regulation requirements and protect data privacy, Google proposed the federated machine learning (FML) \cite{mcmahan2016communication,konevcny2016federated,OptimizationKone2016Federated}. 
The key concept of their work is to train models without integrating the data together in one place, and no raw data would be exposed to other parties but fully secured and under users' own control. 
Different from Google, we are interested in the business situations that several similar and small size companies such as regional banks want to build joint models together to solve a common business problem, e.g. intelligent loan application approval. 
Inspired by this, we proposed a novel privacy-preserving federated machine learning model, entitled \textit{Federated Extra-Trees (FET)}. 
Based on it, a secure multi-institutional machine learning system was developed to support real-world applications accurately, robustly and safely. We have four major contributions:

\begin{itemize}[leftmargin=*]
\item \textbf{Data privacy was secured} by embedding local differential privacy (LDP) into the Federated Extra-Trees, as well as establishing a third-party trusty server to coordinate and monitor the entire modeling process. And the mathematical proof is provided to illustrate that our model satisfies local differential privacy. 

\item \textbf{Accuracy} was guaranteed under the horizontal federated scenarios. Although LDP and randomness were introduced in several stages, our model was proved to maintain the same level of accuracy as the non-federated approach that brings the data into one place.

\item \textbf{High efficiency} was achieved with the random tree building process and our model is robust to the complicated network environments. Only necessary and privacy-free modeling information was exchanged and the message size was reduced to a minimum.

\item The total solution is \textbf{practical, extensible, scalable and explainable} to handle the \textit{data island} problem and can be easily deployed for real-life applications. 

\end{itemize}
% ==========================================================
\section{Related Work and Preliminaries\label{sec:pre}}
\subsection{Federated Learning}
In the work of \cite{yang2019federated}, they have provided a clear definition for the federated machine learning and how it distinguishes from other subjects, such as distributed machine learning, secure multi-party computation, etc. Generally, it can be categorized into three types, horizontal federated learning, vertical federated learning and federated transfer learning. The horizontal FML \cite{mcmahan2016communication,konevcny2016federated,OptimizationKone2016Federated,chen2018federated,yao2018differential} is focused on solving problems with data from different sample space but same feature space. The vertical FML \cite{hardy2017private,cheng2019secureboost,liu2019federated} is the opposite, which works on problems with the same sample space but different feature space. The federated transfer learning \cite{liu2018secure} is mainly about tasks that data from different sources are overlapped in both sample and feature space, but still largely different from each other. Currently, most FML methods were developed to solve problems under horizontal scenarios. Google applied FML in applications such as on-device item ranking and next word prediction \cite{Bonawitz2019}. In the work of \cite{smith2017federated} the researchers applied FML to solve multi-task problems and a novel federated recommender system was proposed in the work of \cite{chen2018federated}. 
% An agnostic federated learning model was presented in the work of \cite{Mohri2019} and it focused on the natural target distribution formed by the mixed clients.

\subsection{Differential Privacy}

Differential Privacy (DP) \cite{dwork2008differential} is a commonly applied privacy-preserving method in federated learning. 
It aims to minimize the possibility of individual identification to ensure user-level privacy \cite{Kairouz2014NIPS}. 
DP has been used extensively in machine learning tasks against privacy inference attacks. 
Existing work mainly focuses on adding perturbations on parameters in the gradient descent algorithms \cite{song2013stochastic,Abadi2016,Geyer2017}. 
% DP methods typically take a balance between privacy and utility. 
% Recent work \cite{Bargav2019} has investigated into the trade-offs for learning tasks. 
Global differential privacy (GDP) and local differential privacy (LDP) are the two main classes of DP. 
The existing approaches in federated learning are mostly GDP, where a trusted curator will apply calibrated noise on the aggregated data to provide differential privacy. 
Conversely, LDP mechanisms, where owners will perturb their data before aggregation, provide better privacy without trusting any third party as a curator. 
LDP applications is on the rise due to its higher privacy and simpler implementation \cite{bhowmick2018protection,AAAIW1816631,chamikara2019local}. 
Formally, LDP is defined as follows \cite{duchi2013local}:
\newtheorem{definition}{Definition}
\begin{definition}[$\varepsilon$-local differential privacy] A randomized algorithm $\mathcal{M}$ is $\varepsilon$-local differential privacy if and only if for any two input tuples $u$ and $u'$ in the domain of $\mathcal{M}$, and for any output $u^{*}$ of $\mathcal{M}$, $Pr[\mathcal{M}(u)=u^{*}] \leq e^{\varepsilon} Pr[\mathcal{M}(u')=u^{*}]$.
\end{definition}
For a complex randomized algorithm with multiple sub-functions, two composition theorems \cite{mcsherry2009privacy} were widely used.
\begin{theorem}[Sequential Composition] If a series of algorithms $\mathcal{M} = \{\mathcal{M}_1,...,\mathcal{M}_P \}$, in which $\mathcal{M}_p$ satisfies ${\varepsilon}_p$-local differential privacy, are sequentially on a dataset, $\mathcal{M}$ will satisfies  $\sum_{p=1}^{P}{\varepsilon}_p$-local differential privacy.
\end{theorem}
\begin{theorem}[Parallel Composition] If a series of algorithms $\mathcal{M} = \{\mathcal{M}_1,...,\mathcal{M}_P \}$, in which $\mathcal{M}_p$ satisfies ${\varepsilon}_p$-local differential privacy, are performed separately on disjoint datasets, $\mathcal{M}$ will satisfies $\underset{1\leq p \leq P}{max} \{\varepsilon_p\}$-local differential privacy.
\end{theorem}

\begin{figure}[!t]
    \centering
\includegraphics[width=\linewidth]{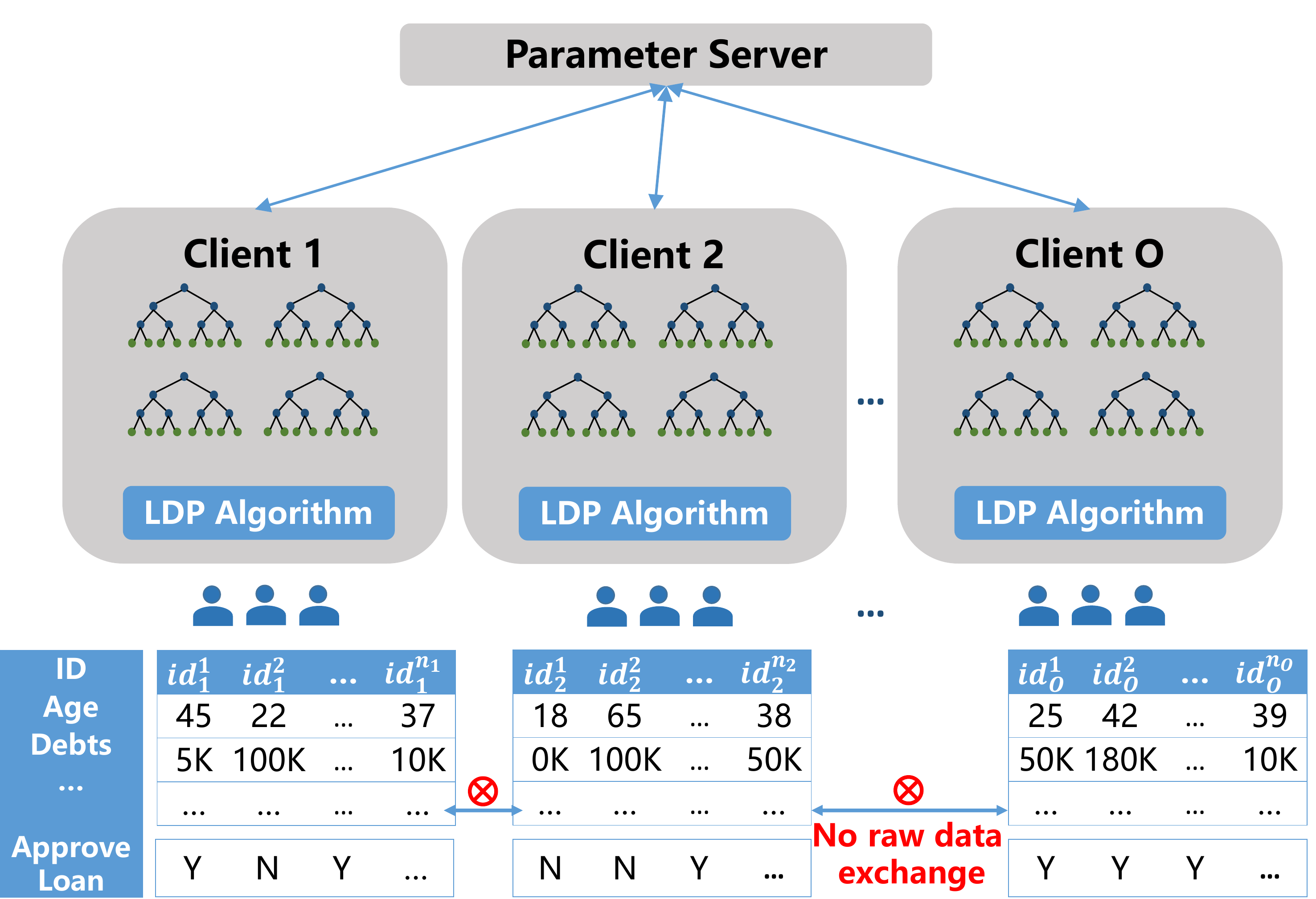}
    \caption{Framework of the Federated Extra-Trees}
    \label{fig:framework}
    \vspace{-.2cm}
\end{figure}

\section{Methodology\label{method}}
\subsection{Learning Scenario}
In our work, we focused on applying Federated Extremely Randomized Trees, abbreviated to Federated Extra-Trees, to solve horizontal distributed data problems, that all data providers have the same attribute set $\mathcal{F}$ but different sample space. Each data provider was considered as one institutional data domain and denoted as $\mathcal{D}_{i}$. The overall data domain is $\mathcal{D} = \{\mathcal{D}_{1} ; \mathcal{D}_{2}; \cdots ; \mathcal{D}_{O}\}$, where $1\leq i \leq O$ and $O$ is the number of institutional domains. On each data domain, we have $\mathcal{D}_{i} = \left(\left(x_i^1, y_i^1\right), \left(x_i^2, y_i^2\right), ..., \left(x_i^{n_i}, y_i^{n_i}\right)\right)$. Here $x$ is the input sample and $y$ is the corresponding label, $(x,y)\in (\mathcal{X},\mathcal{Y})$ and $n_i$ is the total number of samples in $\mathcal{D}_{i}$. We have deployed a master machine as the parameter server to coordinate the entire modeling process and assigned each institutional domain one client machine. Since we are trying to build FML models jointly on different organizations, $O$ is usually small and in our work, we only consider situations when $O <= 10$. For more parties involved, the algorithm design could be much more different. 

% The notations appeared in this paper are also shown in Table \ref{Notations1}.

\subsection{Problem Statement}
The formal statement of the problem is given as below:
\begin{itemize}[leftmargin=*]
\item \textbf{Given:} Institutional data domain on each client.
% \item \textbf{Given:} Institutional data domain $D_i$ on each client $i$, $ 1<= i <= O$.
\item \textbf{Learn:} Privacy-preserved Federated Extra-Trees.
\item \textbf{Constraint:} The performance (accuracy, f1-score, etc) of the Federated Extra-Trees must be comparable to the non-federated approach.
\end{itemize}

% ==========================================================
\subsection{Framework Overview}
% The framework of the Federated Extra-Trees is based on the previous work of Extra-Trees \cite{Geurts2006ExtremelyRT}, bagging \cite{breiman1996bagging} and local differential privacy (LDP) \cite{duchi2013local}. The Extra-Trees model is naturally suitable for distributed machine learning because of its concise algorithm design and limited computation complexity, and it is an ideal solution for the federated scenarios. Our FET model is able to solve classification problems and without applying LDP it also supports regression tasks.

In our work, we carefully extended the Extra-Trees \cite{Geurts2006ExtremelyRT} to suit the horizontal federated scenarios with full consideration of privacy issues by applying local differential privacy (LDP) \cite{duchi2013local}. In Extra-Trees, the optimal splitting under a certain feature is randomly selected instead of being calculated. With the idea of bagging \cite{breiman1996bagging}, a forest can accommodate the errors caused by randomness in the single trees. Our model has provided a concise algorithm design and the computation complexity is limited to the minimum, and the training speed is greatly improved. Our model is able to solve classification problems and without applying LDP it also supports regression tasks.

As shown in Figure \ref{fig:framework}, assume we want to build an intelligent system to automatically decide if we should approve or reject the loan application. We have $O$ clients and each of them provides information on their own loan application records. Before the modeling, we first applied LDP to transform the clients’ labels into encoded binary strings, then all clients work together to build a complete federated forest that is available for subsequent use on every client. During the training, no raw data such as Gender, Age or others would be exposed. When modeling is finished, the model will be saved locally for inference use and no communication is necessary.

\subsection{Algorithms}

\begin{algorithm}[!b]
% \LinesNumbered % 显示行号
\SetKwInOut{Input}{Input}\SetKwInOut{Output}{Output}
\SetKwFunction{Tree}{build\_tree}
\Input{Training set $\mathcal{D}_i$ of client $i$,\ feature set $\mathcal{F}$ }
\Output{A Federated Extra-Tree}
\caption{Federated Extra-Tree -- Client \label{alg1}}
\SetKwProg{TreeGenerateAlg}{Function}{}{end}
\setlength{\baselineskip}{1.5em}

    % selected feature names and sample indices from master;\\
    {$\mathcal{S}_{i}$ $\leftarrow$ {subsample of $\mathcal{D}_i $ on client $i$};}
    
        \TreeGenerateAlg{\Tree{$\mathcal{S}_{i},\mathcal{F}$}}
        {
            \If{ $stopping\_condition$ \rm \textbf{is true}}
            {
                % \rm{Mark current node as leaf node;}\\
                % $\mathcal{Y}_{i}\leftarrow$ Encoded labels of remaining $\mathcal{S}_i$;\\
                \rm{Send $Sum_{i}$ to master;}\\
                % \rm{Receive \textit{perturbed leaf node} from master;}\\
                \rm{Receive \textit{leaf labels} from master;}\\
                \Return{leaf node;}
            }
            \rm Receive feature candidate set $\mathcal{F^{*}}$ from master; \\
            \For{\rm{\textbf{each feature $f_{i,j}\in \mathcal{F^{*}}$}}}
            {   ${v}^{min}_{i,j},{v}^{max}_{i,j}\leftarrow $ \rm{local min \& max value of $f_{i,j}$;}\\
            Random pick $v_{i,j}^* \in ({v}^{min}_{i,j},{v}^{max}_{i,j})$ and send to master;\\
                % \rm Send ${v}^{min}_{i,j},{v}^{max}_{i,j}$ to master;\\
                \rm Receive split threshold $v_{j}^{*}$ from master; \\
             $\mathcal{S}_{i_{L},j},\mathcal{S}_{i_{R},j}$ $\leftarrow$ Split $\mathcal{S}_{i}$ by $v_{j}^{*}$ of feature $f_j$;\\
                % $\mathcal{Y}_{i_{L},j}^{e}, \mathcal{Y}_{i_{R},j}^{e}\leftarrow$ Encoded  corresponding labels;\\
                $Sum_{i_{L},j}, Sum_{i_{R},j}\leftarrow$label aggregation;\\
                % Send $\mathcal{Y}_{i_{L},j}^{e}, \mathcal{Y}_{i_{R},j}^{e}$ to master;\\
                Send $Sum_{i_{L},j}, Sum_{i_{R},j}$ to master;\\
                }
            Receive global best split feature $f^*$ and value $v^*$;\\
            $\mathcal{S}_{i_{L},j_*},\mathcal{S}_{i_{R},j_*}$ $\leftarrow$ Split $\mathcal{S}_{i}$ by $v^*$ of feature $f^*$;\\
            left\_subtree $\leftarrow$ \Tree{$\mathcal{S}_{i_{L},j_*},\mathcal{F}$};\\
            right\_subtree $\leftarrow$ \Tree{$\mathcal{S}_{i_{R},j_*},\mathcal{F}$};\\
            \Return{tree node}
        }
    Append current tree to forest;\\
\end{algorithm}

% 这部分的内容应该按照总分的框架来写：
% 1. 首先介绍一下这个算法是干什么的，以前有哪些工作做了类似的工作，接着展开一些算法中的重点。(极端树->分布式->隐私保护)
% 2.1 Rappor 差分隐私，可以阐述一些参数的设置
% 2.2 随机筛选分裂点的阐述
% 2.3 停止建树的条件的阐述
%
% We carefully extended the Extra-Trees to a distributed version with full consideration of privacy issues. In Extra-Trees, the optimal splitting under a certain feature is randomly selected instead of being calculated. With the idea of bagging, a forest can accommodate the errors caused by randomness in the single trees. Previous work has proved the scheme to be as accurate as of the ordinary tree models, and the training speed is greatly improved \cite{Geurts2006ExtremelyRT}. Therefore, this method is naturally suitable for federated scenarios. 
In this part, we will give a detailed introduction to our model. The training process of clients and master is described in Algorithm \ref{alg1} and \ref{alg2}. All participants, including the master and clients, share the same feature set. 
% Initially, each client samples data from its own training set $\mathcal{D}_i$.
% We constructed the trees in the depth-first search sequence. The key steps of building a tree are as follows.
% We constructed the trees in the best-first search sequence.
The key steps of building a tree are as follows.

\paragraph{Stopping criterion.} Before creating a new tree node, participants will check if the stop conditions have been satisfied. Here we adopted a CART-tree \cite{breiman2017classification} like design. The stopping conditions are set by a maximum threshold for the depth of trees, a limit on the number of remaining samples in leaf nodes as well as other corner conditions. 

\begin{algorithm}[!t]
% \LinesNumbered % 显示行号
\SetKwInOut{Input}{Input}\SetKwInOut{Output}{Output}
\SetKwFunction{Tree}{build\_tree}
\SetKwFunction{Split}{pick\_random\_split}
\Input{Feature set $\mathcal{F}$}
\Output{A Federated Extra-Tree}
\caption{Federated Extra-Tree -- Master\label{alg2}}
\SetKwProg{TreeGenerateAlg}{Function}{}{end}
\setlength{\baselineskip}{1.5em}

    % selected feature names and sample indices from master;\\
        \TreeGenerateAlg{\Tree{$\mathcal{F}$}}
        {
            \If{ $stopping\_condition$ \rm \textbf{is true}}
            {
                % \rm{Mark current node as leaf node};\\
                % \rm{Receive $\mathcal{Y}_{i}^{e}$ from client $i = 1,...,M$};\\
                \rm{Receive $Sum_{i}$ from client $i = 1,...,O$};\\
                % \rm{Send global $Sum$ to clients as \textit{leaf node}};\\
                \rm{Send global $Sum$ to clients as \textit{leaf labels}};\\
                \Return{leaf node};
            }
            $\mathcal{F^{*}}\subset \mathcal{F}\leftarrow$ Randomly-chosen feature subset;\\
            \rm Send feature candidate set $\mathcal{F^{*}}$ to clients; \\
            \For{\rm{\textbf{each feature $f_{j}\in \mathcal{F^{*}}$}}}
            {   \rm{Gather ${v}^{*}_{i,j}$ from client $i = 1,...,O$};\\
            % ${v}^{min}_{j},{v}^{max}_{j}\leftarrow$\rm{global min \& max value of $f_j$};\\
            Pick a random split threshold $v_{j}^{*}\in \left(min({v}^{*}_{i,j}), max({v}^{*}_{i,j})\right)$;\\
            Broadcast $v_{j}^{*}$ to all clients;\\
            % Gather $\mathcal{Y}_{i_L,j}^{e}, \mathcal{Y}_{i_R,j}^{e}$ from client $i=1,...,O$;\\
            Gather $Sum_{i_{L},j}, Sum_{i_{R},j},\ i=1,...,O$;\\
            $\mathcal{C}_{L,j},\mathcal{C}_{R,j}\leftarrow $ estimated global label counts;\\
            Calculate $Gini\_Gain(f_j)$ with ($\mathcal{C}_{L,j},\mathcal{C}_{R,j}$);\\}
            $f^* = \mathop{\arg\max}\limits_{f_j}Gini\_Gain(\mathcal{F}^*)$ ;\\
            Broadcast the global best split feature $f^*$ and the corresponding split threshold $v^*$;\\
            left\_subtree $\leftarrow$ \Tree{$\mathcal{F}$};\\
            right\_subtree $\leftarrow$ \Tree{$\mathcal{F}$};\\
            \Return{tree node}
        }
    Append current tree to forest;\\

\end{algorithm}

\paragraph{Random feature and threshold selection.} Master is responsible for coordinating the collection of information from clients and decide which feature to use on a node. We inherited the randomness solution in Extra-Trees and extended it to the entire process of feature selection. Experiments in Section \ref{experiment} have shown that the randomization does not necessarily lead to loss of precision. To create a new tree node, the master would randomly extract a candidate feature set $\mathcal{F^{*}}\subset \mathcal{F}$, and send it to all clients. Each client $i$ randomly picks a value $v_{i,j}^{*}$ between the local minimum and maximum value of feature $j$, then send it to the master. The master collects $v_{i,j}^{*}$ and arbitrarily picks a value $v_{j}^{*}$ between $min(v_{i,j}^{*})$ and $max(v_{i,j}^{*})$ as the split threshold for each feature, then broadcasts the values to all clients. In this way, the true local range of features on each client will not be revealed to the master.

Clients would split local data temporarily into left and right subtrees according to the received feature threshold, then sending perturbed information of the data labels to master. This process also increases the randomness. Receiving all the information of local subsets, the master would aggregate the data to calculate a $Gini\_Gain$ value for the feature. Feature $f^*$ with the maximum $Gini\_Gain$ will be chosen as the best split feature for the current node. Clients should record the subsets for each feature. When the split feature $f^*$ is finally determined, they could use the corresponding subsets directly to avoid repeated calculations.

\begin{figure*}[!t]
    \centering
    \includegraphics[width=.95\linewidth]{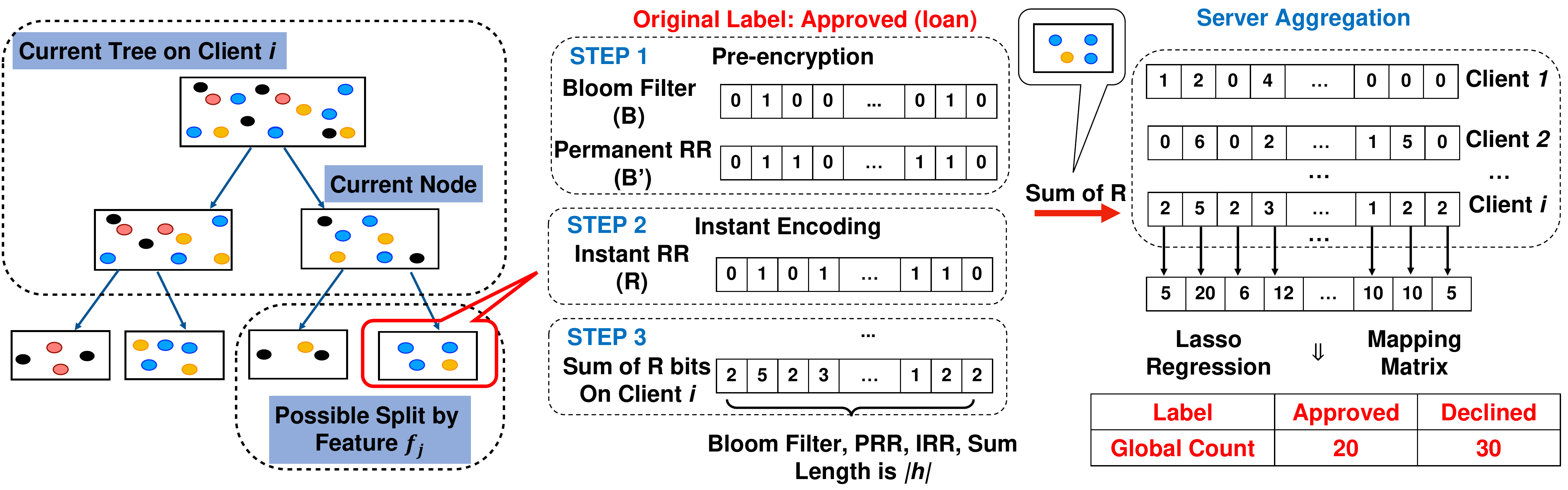}
    \caption{Privacy-Preserving Methodology in Federated Extra-Trees}
    \label{fig:random_response}
    % \vspace{-.2cm}
\end{figure*}

% 这里应当阐述决策树与RAPPOR是如何结合在一起的以及从差分隐私的角度论述该做法的和
\subsection{Privacy Preserving Methods}
In Figure \ref{fig:random_response}, the dots with different colors represent different label classes. 
When the building process proceeds to a new node, the master 
% sends a randomly-picked feature set to the clients and assigns corresponding threshold values according to their response. 
% Master then 
needs to know the global label distribution of split data under a feature threshold, so that it can calculate the $Gini\_Gain$ value for the feature.

The algorithm should neither reveal the category of a single user nor compromise the specific distribution of categories on a client. Here we modify an aggregation algorithm, which was first proposed by Google \cite{Erlingsson_2014} for crowd-sourcing business and proven to be locally differential private. We implement a multi-layer mechanism, including one Bloom Filter layer and two separate random-response based layers. Bloom Filter \cite{broder2004network} is a randomized structure for representing a set in a space-efficient way. It adds extra uncertainty for user identification and compacts large data to reduce the communication traffic in federated scenarios.

\textbf{Step 1: }Two fixed layers are set before the tree is created. 
% The corresponding labels $\mathcal{Y}_i$ for samples in $\mathcal{D}_i$ will be pre-encoded into binary Bloom Filter strings. 
For the $k$-th sample in $\mathcal{D}_i$, its label $y_i^k$ maps to Bloom Filter $B^k_i$ of size $h$ using several hash factions.
% We set $h=32$ and the bits corresponding to four hash function results of $y_i^k$ are set to 1.
The Bloom Filter strings are then encrypted as permanent random responses (Permanent RR), i.e., the second layer. Each bit in $B_i^k$ would maintain the original value with probability $pr$; otherwise it will be replaced by 0 or 1 with equal probability $1/2(1-pr)$. 
% , as shown in Equation \ref{bt},
% \begin{equation}
% B_{t}^{'}=\left\{\begin{matrix}
% B_t, & P=\textit{f}\\ 
% 1, & P=\displaystyle1/2(1-\textit{f}) \\ 
% 0, & P=\displaystyle1/2(1-\textit{f})
% \end{matrix}\right.
%B_{i,t}^{k'}=\left\{
%             \begin{array}{lr}
%             B^k_{i,t}, & P=p \\
%             1, & P=\displaystyle1/2(1-p) \\
%             0, &  P=\displaystyle1/2(1-p)
%             \end{array}
% \right. 
% \label{bt}
% \end{equation}
% where $t=1,2,\dots,h$. 
%We set $p$ to 0.5 in our mechanism, while the value could be tuned to control the perturbation performance. 
% In the whole process of building trees, the original labels will never be used or transmitted. Instead, the permanent binary string $B^{'}$ is reused as a substitute.

\textbf{Step 2: }For each feature selection process, another layer of temporary perturbation shall be added on $B^{k'}_i$, i.e., an instant random response string (Instant RR) denoted as $R_k$. Each bit $R_{i,t}^k$ is set to 1 with a certain probability, as is shown in Equation \ref{Pst}.
\begin{equation}
Pr(R_{i,t}^k=1)=\left\{
             \begin{array}{lr}
             \xi, & if\ {B_{i,t}^{k'}}=1 \\
             \zeta, & if\ {B_{i,t}^{k'}}=0 
             \end{array}
,t=1,2,\dots,h    
\right.
\label{Pst}
\end{equation}
% We set $\xi$ to 0.8 and $\zeta$ to 0.2, so as to add a reasonable interference. Here $\xi$ and $\zeta$ are independent of each other.

\textbf{Step 3: }
Client $i$ adds the local values along the bit position, as shown in Equation \ref{Rst}.
\begin{equation}
    Sum_{i,t} = \sum_{k=1}^{n_i}R_{i,t}^k,t=1,2,\dots,h 
\end{equation}
\label{Rst}
where $n_i$ is the number of users on client $i$.
%adds each bit of the local $R$ values for users in $\mathcal{D}_i$, 
%and sends the result array $Sum_i$ to master. 

\textbf{Aggregation: }As is shown in Figure \ref{fig:random_response}, the master will aggregate the received results into $Sum$, with the count of each bit being:
\begin{equation}
Sum_t = \sum_{i=1}^{M}Sum_{i,t},t=1,\dots,h
\end{equation}
With the label space mapped into $B_1,B_2,\dots,B_L$, master estimates the overall label counts using linear estimation methods such as Lasso Regression. 

In this scheme, the Permanent RR is already fixed, and the instant perturbation is calculated at an individual level without trusting any third party as a curator. 
%It has been proven to be effective in protecting privacy \cite{Erlingsson_2014}. 
This LDP method also applies to other models that use statistics as an intermediate value.
For comparisons, we also adopt GDP in Federated Extra-Trees, i.e., the clients add a disturbance to their local labeling statistics and the master sums up the received statistics directly for further calculation. In this case, the clients must be fully trusted to be responsible for ensuring the data privacy of end-users. 
%Moreover, GDP causes larger fluctuations in the accuracy of results. 
In the experimental part, we carry out a GDP-based method using the Laplace mechanism.

\subsection{Privacy Analysis}
In this part, we will provide an analysis of the privacy level of our proposed algorithm.
\newtheorem{corollary}{Corollary}
\begin{corollary}
The output of $p$-th tree on $i$-th client satisfies $\varepsilon_{ip}$-local differential privacy.
\end{corollary}
\begin{proof}
After the private data is perturbed by LDP algorithm, every query that acts on the dataset satisfies $\varepsilon_{node}$-local differential privacy. Considering the structure of random decision trees, different nodes on every layer own disjoint datasets, which satisfies parallel composition. Thus, the maximum privacy budget will not be larger than $\varepsilon_{ip}=\varepsilon_{node} *(depth + leaf)$.
\end{proof}

\begin{corollary}
The \emph{FET} satisfies $\varepsilon$-local differential privacy.
\end{corollary}
\begin{proof}
There are two views on $\emph{FET}$, random decision trees' view and participating clients' views. In the previous perspective, $P$ mutually independent random decision trees of one client act on the same data set, which satisfies sequential composition. In the latter perspective, the decision trees of $O$ clients act on $O$ disjoint data sets, which satisfies parallel composition. With these two composition theorems, we have
\begin{align*}
Pr[\mathcal{M}(t)=t^*] & = \prod_{p=1}^{P} Pr[\mathcal{M}_p(t)=t^*]\\
& = \prod_{p=1}^{P} \prod_{i=1}^{O} Pr[\mathcal{M}_{ip}(t)=t^*] \\
& \leq \prod_{p=1}^{P} e^{\underset{1\leq i \leq O}{max} \{\varepsilon_{ip}\}} Pr[\mathcal{M}_{p}(t')=t^*] \\ 
& \leq  e^{\sum_{p=1}^{P}{\underset{1\leq i \leq O}{max} \{\varepsilon_{ip}\}} } Pr[\mathcal{M}(t')=t^*] 
\end{align*}
Therefore, we proved that our proposed \emph{FET} model satisfies $\varepsilon=\sum_{p=1}^{P}{\underset{1\leq i \leq O}{max} \{\varepsilon_{ip}\}}$
-local differential privacy.
\end{proof}

% 1. Why we use LDP
% 强调不需要可信的第三方 以及 误差不会累积

%In previous work[x,x,x], data will be collected by a trusted third %party. And during the transmission process of data, all possible %methods, like HE and MPC, will be used to protect the security of %data. However, kind of this way will pull the data security boundary %beyond the control of data owner. Based on such factors, we add %Local DP into the construction process of decision tree. In this %setting, all nodes will perturb data locally and a collaborative model %will be built from these distributed and privacy-preserved data.
%In addition to data security factors,compared to centralized %differential privacy, Local DP avoid the accumulation of privacy cost %to the some extent.

% 2. How to use it\textsc{
%For each client $c$, the corresponding label $\mathcal{Y}_c$ will be %encoded as Bloom Filter and perturbed by permanent randomized r%esponse locally. The perturbed label $\mathcal{Y}^e_c$ will replace %the real label value.}

\section{Experimental Studies\label{experiment}}
\subsection{Experimental Setup}
To verify the effectiveness of our algorithm and the utility of privacy-preserving methods, we designed comparative experiments of the following four algorithms:
\begin{itemize}[leftmargin=*]
	\item \textbf{Extra-Trees (ET)}: The non-federated implementation of the extremely randomized trees.
	\item \textbf{Federated Extra-Trees (FET)}: Our federated extremely randomized trees without perturbations on the input data.
	\item \textbf{FET-LDP}: Our Federated Extra-Trees with local differential privacy (random response mechanisms).
	\item \textbf{FET-GDP}: Our Federated Extra-Trees with global differential privacy (Laplace mechanisms).
\end{itemize}

We have carried out tests on various UCI datasets \cite{Dua:2019} and MIMIC dataset \cite{johnson2016mimic} with a different number of samples and attributes for classification tasks. Both numerical and categorical data were considered. The MIMIC dataset was processed by following the work of \cite{huang2018loadaboost}.
% We followed the method of \cite{huang2018loadaboost} to process MIMIC data and get the corresponding evaluation dataset.
\begin{table}
\renewcommand{\arraystretch}{1.1}
\centering
\begin{tabular}{lp{2cm}p{1.5cm}p{1cm}p{1cm}}
\hline
Dataset  & Size & Feature Size & Label Size
\\ \hline
{Spambase}             & 4600 &     57            &     2           
\\\hline
{Credit-card}          & 30000 &        23         &      2 
\\\hline
{MIMIC} & 35120 &       2168          &     2         
\\\hline

{Waveform}             & 5000 &            21     &      3         \\\hline
{Letter-recognition}   & 20000 &          16     &26\\\hline
{KDDCUP 99}           & 4000000 &        42         & 23              \\\hline
\end{tabular}
\caption{Dataset Details}
\label{dataset}
% \vspace{-.2cm}
\end{table}
Each dataset was divided into a training set and a test set. The division ratio of the training set and the test set is 8:2, and for KDD Cup 99 dataset with a large amount of data, the ratio is 99:1. The evaluation criteria are the accuracy and F1 score. For multi-classification, the F1 score refers to the Micro F1 value.

\begin{table*}[!t]
\renewcommand{\arraystretch}{1.1}
% \small
\centering
\begin{tabular}{cccccc}
\hline
\multicolumn{6}{c}{Binary Classification}                                                                                              \\ \hline
Metric                    & Dataset         & ET           & FET       & FET-LDP      & FET-GDP \\ \hline
\multirow{3}{*}{Accuracy} 
                          & Spambase         &  0.934$\pm$0.005      &  0.932$\pm$0.009  & 0.920$\pm$0.009  &  0.919$\pm$0.013 \\ \cline{2-6}
                          & Credit-Card      &  0.805$\pm$0.002      &  0.814$\pm$0.004  & 0.815$\pm$0.003  &  0.811$\pm$0.012 \\ \cline{2-6}
                          & MIMIC &  0.639$\pm$0.003    &  0.641$\pm$0.011  & 0.645$\pm$0.008  &  0.638$\pm$0.010 \\ \cline{2-6} 
                          \hline
\multirow{3}{*}{F1 Score} 
                          & Spambase         &  0.943$\pm$0.004      &  0.934$\pm$0.021  & 0.920$\pm$0.009  &  0.915$\pm$0.018 \\ \cline{2-6}
                          & Credit-Card      &  0.885$\pm$0.001      &  0.890$\pm$0.002  & 0.852$\pm$0.040  &  0.889$\pm$0.005 \\ \cline{2-6}
                          & MIMIC &  0.776$\pm$0.001    &  0.781$\pm$0.006  & 0.645$\pm$0.008  &  0.776$\pm$0.005 \\ \cline{2-6} 
                          \hline
% \multirow{3}{*}{Time} 
%                           & Spambase         &  -      &    &   &   \\ \cline{2-6}
%                           & Credit-Card      &  -      &    &   &   \\ \cline{2-6}
%                           & MIMIC &      &    &   &  \\ \cline{2-6} 
%                           \hline
\multicolumn{6}{c}{Multiclass Classification}                                                                                          \\ \hline
Metric                    & Dataset         & ET           & FET       & FET-LDP      & FET-GDP \\ \hline
\multirow{3}{*}{Accuracy} 
                          & Waveform         &  0.842$\pm$0.005      &  0.886$\pm$0.010  & 0.882$\pm$0.011  &  0.876$\pm$0.012 \\ \cline{2-6}
                          & Letter           &  0.953$\pm$0.002      &  0.971$\pm$0.004  & 0.940$\pm$0.007  &  0.972$\pm$0.003 \\ \cline{2-6}
                          & KDD CUP 99       &  0.994$\pm$0.001      &  0.994$\pm$0.002  & 0.993$\pm$0.002  &  0.993$\pm$0.002 \\ \cline{2-6} 
                          \hline
\multirow{3}{*}{F1 Score} 
                          & Waveform         &  0.842$\pm$0.005      &  0.886$\pm$0.010  & 0.882$\pm$0.011  &  0.876$\pm$0.012 \\ \cline{2-6}
                          & Letter           &  0.953$\pm$0.002      &  0.971$\pm$0.004  & 0.940$\pm$0.007  &  0.972$\pm$0.003 \\ \cline{2-6}
                          & KDD CUP 99       &  0.994$\pm$0.001      &  0.994$\pm$0.002  & 0.993$\pm$0.002  &  0.993$\pm$0.002 \\ \cline{2-6} 
                          \hline        
\end{tabular}
\caption{Experimental Results\label{experiments1}}
\vspace{-.2cm}
\end{table*}

The experiments of the classical ET are performed on a single node with the entire datasets. For the overall-performance evaluation, the distributed experiments are conducted on two client nodes and one master node, which is the minimum size of a parameter-server-based federated learning system. A dataset is randomly divided into several subsets and distributed to the clients to simulate the horizontal scenario. The sample space between the clients does not intersect. We also provide supplementary experiments to analyze the influence of the number of clients, the number of trees, and the maximum tree depth on the performance.

\subsection{Overall Performance}
The overall performance is shown in Table \ref{experiments1}. Each experiment was repeated for 30 times, and the mean and variance of the accuracy and F1 score are given. We can see these points from the Table \ref{experiments1}:
\begin{itemize}[leftmargin=*]

% \item \textbf{Lossless-ness (accuracy):} On different scales of datasets, our algorithms have achieved comparable results to the non-federated Extra-Trees. For both binary and multiclass classification tasks, our FET method performs even better than the basic ET model on some datasets. 

\item \textbf{Accuracy:} On different scales of datasets, our algorithms have achieved comparable results to the non-federated Extra-Trees. For both binary and multiclass classification tasks, our FET method performs even better than the basic ET model on some datasets. 

% \item \textbf{Utility of privacy-preserving methods:} As observed in most tests, methods with perturbations (FET-LDP \& FET-GDP) have brought a loss to the accuracy, but overall this loss is acceptable. The performances of both FET-LDP and FET-GDP are close to that of FET. This is because the framework of Extra-Trees is inclusive and can accommodate the fluctuations caused by the perturbations. The other reason is that the perturbed data is used to select features, not directly involved in numerical modeling.

\item \textbf{Utility of privacy-preserving methods:} As observed in most tests, DP-based methods have brought a loss to the accuracy, but overall this loss is acceptable. The performances of both FET-LDP and FET-GDP are close to that of FET. This is because the framework of Extra-Trees is inclusive and can accommodate the fluctuations caused by the perturbations.
The other reason is that the perturbed data is used to select features, not directly involved in numerical modeling.

\item \textbf{Stability: } The variance values show the stability of FET-series algorithms compared with ET. Our LDP-based approach presented relatively smaller variances than the GDP method and maintained almost the same stability as the privacy-free FET. The datasets we used have covered a wide range of data volume and feature types, which besides shows the adaptability of our algorithm.

\end{itemize}
To analyze the impacts of the number of clients and tree parameters on the models, we experimented with the algorithmic efficiency of the three federated algorithms with different parameters on two datasets: the binary-class Credit-card dataset and the multi-class Letter-recognition dataset. 
\subsection{Effect of Client Numbers}
In this set of experiments, the number of trees and the maximum tree depth were fixed to 20. We have randomly divided both datasets into 9 folds, and each was placed on one client. As is shown in Figure \ref{fig:depth} (a,d), when we have more clients modeling together, there is a constant increase in the accuracy. This has supported our vision that by uniting more institutions a better modeling performance could be achieved. LDP method was also applied to protect the data privacy of each client.

% And because the LDP method was introduced and a parameter server was deployed, each client was blind to others and the data privacy was not compromised.

% Even the message was hacked, it is difficult to restore the original statistical information and steal any valuable information from it.

\subsection{Effect of Tree Settings}
We tested the effects of the number of trees and the maximum tree depth respectively. When experimenting with the effect on the number of trees, the maximum tree depth is set to 20, and vice versa.
By observing the change of the fold lines in Figure \ref{fig:depth}, we could find:
\begin{itemize}[leftmargin=*]
\item \textbf{The number of trees: }The accuracy rate has been greatly improved from a single tree to multiple trees, which demonstrates the advantages of forest structure.
% as we mentioned in Section \ref{method}.
However, in the comparison of multiple trees, the increase in the tree numbers has minimal impact on the results.
\item \textbf{The maximum tree depth: }The maximum tree depth has a greater influence on the results. The accuracy of Federated Extra-Trees models is rising continuously as the tree depth threshold grows. When the maximum depth reaches 20 or so, the model converges. 

% \vspace{-.}

\end{itemize}
\begin{figure}[!t]
    \centering
    \includegraphics[width=1\linewidth]{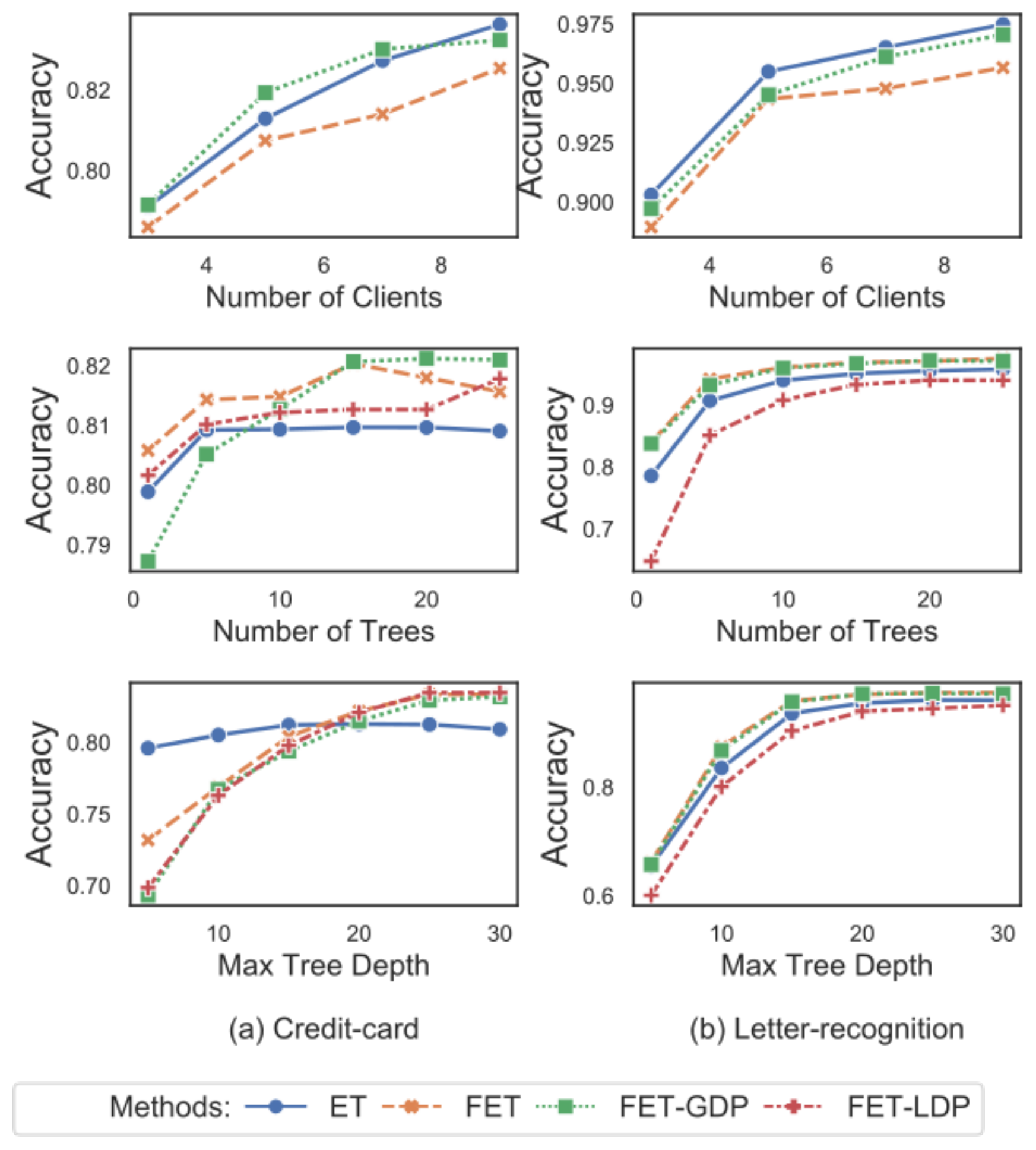}
    \caption{Experimental Results on Parameter Impacts}
    \label{fig:depth}
    \vspace{-.3cm}
\end{figure}
%
% ==========================================================
\section{Conclusions}
% In this paper, we proposed a novel privacy-preserving federated machine learning method, called Federated Extra-Trees, which achieves a lossless performance on the modeling accuracy and protects the data privacy. We also developed a secure multi-institutional federated learning system which allows the modeling task can be jointly processed across different clients with the same attribute sets but different user samples. The raw data on each client will never be exposed, and only limited amount of intermediate modeling values were exchanged to reduce the communication and secure the data privacy. The introduction of local differential privacy and a third-party trusted server strengthens privacy protection and makes it impossible to backdoor the actual statistical information from the clients. We set up multiple clients to simulate the real-world situations and performed experiments on the public UCI datasets. The experimental results presented a superior performance for the classification tasks, and by comparing to the non-federated approach that requires data gathered in one place. We also proved that the introduction of local differential privacy does not affect the overall performance. The efficiency and robustness of our proposed system were also verified. To summarize, the Federated Extra-Trees successfully solved the \textit{data island} problem and provided a brand new approach to protect the data privacy while realizing the cross-institutional collaborative machine learning, and it is strong practical for real-world applications.

In this paper, we proposed a novel privacy-preserving federated machine learning method, called Federated Extra-Trees, which achieves competitive performance on the modeling accuracy and protects the data privacy. We also developed a secure multi-institutional federated learning system that allows the modeling task can be jointly processed across different clients with the same attribute sets but different user samples. The raw data on each client will never be exposed, and only a limited amount of intermediate modeling values were exchanged to reduce the communication and secure data privacy. The introduction of local differential privacy and a third-party trusted server strengthens privacy protection and makes it impossible to backdoor the actual statistical information from the clients. We set up multiple clients to simulate real-world situations and performed experiments on public datasets. The experimental results presented a superior performance for the classification tasks, and there was no loss on the modeling accuracy by comparing to the non-federated approach that requires data gathered in one place. We also proved that the introduction of local differential privacy does not affect the overall performance. The efficiency and robustness of our proposed system were also verified. To summarize, the Federated Extra-Trees successfully solved the \textit{data island} problem and provided a brand new approach to protect the data privacy while realizing the cross-institutional collaborative machine learning, and it is strong practical for real-world applications.

% \newpage

\small

\end{document}